\documentclass{article}

\usepackage{PRIMEarxiv}

\usepackage[utf8]{inputenc}
\usepackage[T1]{fontenc}
\usepackage{mathptmx}                % Times New Roman
\usepackage{amsmath,amssymb}
\usepackage{graphicx}
\usepackage{booktabs}
\usepackage{longtable}
\usepackage{array}
\usepackage{tabularx}
\usepackage{multirow}
\usepackage{xcolor}
\usepackage{float}
\usepackage{caption}
\usepackage{subcaption}
\usepackage{enumitem}
\usepackage{siunitx}
\usepackage{listings}
\usepackage{tikz}
\usepackage{url}

\usepackage{amsthm}
\usetikzlibrary{shapes,arrows.meta,positioning,fit,calc}
\usepackage{algorithm}
\usepackage{algorithmic}
\usepackage{bbm}

\newtheorem{theorem}{Theorem}
\newtheorem{proposition}{Proposition}
\newtheorem{definition}{Definition}
\newtheorem{remark}{Remark}

\usepackage[numbers,sort&compress]{natbib}

\usepackage{hyperref}
\hypersetup{colorlinks=true, linkcolor=blue!60!black, citecolor=blue!60!black, urlcolor=blue!70!black}

\title{Complex-Valued Unitary Representations as Classification Heads for Improved Uncertainty Quantification in Deep Neural Networks
\thanks{\textit{{This work is submitted for review to Neural Networks.}}} 
}

\author{
  A. A. Jafari\\
  University of Tartu\\
  Tartu, Estonia\\
  \texttt{akbar.anbar.jafari@ut.ee}\\
  \and
  C. Ozcinar\\
  University of Tartu\\
  Tartu, Estonia\\
  \texttt{chagri.ozchinar@ut.ee}\\
  \and
  G. Anbarjafari \\
  3S Holding O\"{U} \\
  Tartu 51011, Estonia\\
  \texttt{shb@3sholding.com} \\
}

\begin{document}
\maketitle

\begin{abstract}
Modern deep neural networks achieve high predictive accuracy but remain poorly calibrated: their confidence scores do not reliably reflect the true probability of correctness. We propose a quantum-inspired classification head architecture that projects backbone features into a complex-valued Hilbert space and evolves them under a learned unitary transformation parameterised via the Cayley map. Through a controlled hybrid experimental design---training a single shared backbone and comparing lightweight interchangeable heads---we isolate the effect of complex-valued unitary representations on calibration. Our ablation study on CIFAR-10 reveals that the unitary magnitude head (complex features evolved under a Cayley unitary, read out via magnitude and softmax) achieves an Expected Calibration Error (ECE) of $0.0146$, representing a $2.4\times$ improvement over a standard softmax head ($0.0355$) and a $3.5\times$ improvement over temperature scaling ($0.0510$). Surprisingly, replacing the softmax readout with a Born rule measurement layer---the quantum-mechanically motivated approach---degrades calibration to an ECE of $0.0819$. On the CIFAR-10H human-uncertainty benchmark, the wave function head achieves the lowest KL-divergence ($0.336$) to human soft labels among all compared methods, indicating that complex-valued representations better capture the structure of human perceptual ambiguity. We provide theoretical analysis connecting norm-preserving unitary dynamics to calibration through feature-space geometry, report negative results on out-of-distribution detection and sentiment analysis to delineate the method's scope, and discuss practical implications for safety-critical applications. Code is publicly available.
\end{abstract}

% keywords can be removed

% ============================================================================
% 1. INTRODUCTION
% ============================================================================
\section{Introduction}
\label{sec:intro}

Deep neural networks have achieved remarkable performance across computer vision, natural language processing, and other domains, yet a persistent shortcoming limits their deployment in safety-critical settings: the confidence scores produced by these models are poorly calibrated \cite{guo2017calibration}. A classifier that reports 90\% confidence on a prediction should be correct approximately 90\% of the time; however, modern deep networks systematically exhibit overconfidence, assigning high probabilities to predictions that frequently turn out to be incorrect \cite{minderer2021revisiting}. In applications such as medical diagnosis \cite{begoli2019need,kompa2021second}, autonomous driving \cite{michelmore2018evaluating}, and active learning \cite{gal2017deep}, unreliable uncertainty estimates can lead to consequential errors.

A substantial body of work has addressed this problem through Bayesian approximations \cite{gal2016dropout,blundell2015weight,maddox2019simple}, ensemble methods \cite{lakshminarayanan2017simple}, post-hoc recalibration \cite{platt1999probabilistic,guo2017calibration}, and alternative loss functions \cite{mukhoti2020calibrating}. While effective to varying degrees, these approaches either incur significant computational overhead (ensembles require training and maintaining multiple models), require held-out calibration sets (temperature scaling), or provide only approximate uncertainty (MC-Dropout). A complementary but under-explored direction asks whether the \emph{representational structure} of the network itself can be designed to produce better-calibrated outputs.

In this work, we investigate a quantum-inspired approach to this question. Drawing on the mathematical framework of quantum mechanics \cite{sakurai2020modern,nielsen2010quantum}, we propose classification heads that operate in complex-valued Hilbert spaces with norm-preserving unitary dynamics. The key insight is not the full quantum formalism per se, but rather a specific geometric property: unitary transformations preserve the $\ell_2$-norm of feature vectors, preventing the unconstrained scaling that contributes to overconfident logits in standard linear classifiers.

Complex-valued neural networks have a rich history \cite{georgiou1992complex,tanaka2013complex,lee2022complex}, with notable advances in deep architectures \cite{trabelsi2017deep,guberman2016complex} and unitary recurrent networks \cite{arjovsky2016unitary,wisdom2016full,jing2017tunable}. Quantum-inspired methods have been explored in information retrieval \cite{sordoni2013modeling}, question answering \cite{zhang2018end}, and tensor network classifiers \cite{stoudenmire2016supervised,huggins2019towards}. However, the specific application of complex unitary representations as \emph{drop-in classification heads} for calibration improvement has not been systematically studied.

Our primary contributions are as follows:

\begin{enumerate}
    \item \textbf{Architectural contribution.} We introduce a family of complex-valued unitary classification heads that can be attached to any pretrained backbone. The architecture projects real-valued features into a complex Hilbert space, evolves them under a learned Cayley-parameterised unitary transformation, and produces class probabilities via either Born rule measurement or magnitude-based softmax readout.
    
    \item \textbf{Controlled experimental methodology.} We employ a hybrid backbone--head design in which a single pretrained backbone is shared across all compared methods. This design eliminates confounds due to differences in feature learning capacity, convergence speed, or parameter count, and isolates the contribution of each head component.
    
    \item \textbf{Key empirical finding.} Through systematic ablation on CIFAR-10 \cite{krizhevsky2009learning}, we demonstrate that the unitary magnitude head achieves ECE of $0.0146$---a $2.4\times$ improvement over standard softmax ($0.0355$) and superior to MC-Dropout ($0.0479$), temperature scaling ($0.0510$), and the full Born rule head ($0.0819$). The result that the Born rule \emph{degrades} calibration while unitary dynamics \emph{improves} it is a novel and practically important finding.
    
    \item \textbf{Human uncertainty alignment.} On CIFAR-10H \cite{peterson2019human}, the wave function head achieves the lowest KL-divergence ($0.336$) to human annotator disagreement distributions, suggesting that complex-valued representations naturally capture the structure of perceptual ambiguity.
    
    \item \textbf{Theoretical analysis.} We provide formal analysis connecting the norm-preserving property of unitary maps to improved calibration through bounds on logit magnitude growth.
    
    \item \textbf{Honest delineation of scope.} We report negative results on out-of-distribution detection and compositional sentiment analysis, clearly delineating when complex unitary heads help and when they do not.
\end{enumerate}

The remainder of this paper is organised as follows. Section~\ref{sec:related} reviews related work. Section~\ref{sec:method} presents the mathematical framework and proposed architecture. Section~\ref{sec:experiments} describes the experimental design. Section~\ref{sec:results} reports results across five experiments with detailed analysis. Section~\ref{sec:discussion} discusses practical implications, limitations, and real-world applicability. Section~\ref{sec:conclusion} concludes.

% ============================================================================
% 2. RELATED WORK
% ============================================================================
\section{Related Work}
\label{sec:related}

\subsection{Calibration in Deep Neural Networks}

\cite{guo2017calibration} demonstrated that modern deep networks are significantly miscalibrated and proposed temperature scaling as a simple post-hoc remedy. Subsequent work has refined calibration measurement \cite{nixon2019measuring,kumar2019verified,naeini2015obtaining} and proposed training-time solutions such as focal loss \cite{mukhoti2020calibrating}, Mixup regularisation \cite{zhang2017mixup,thulasidasan2019mixup}, and proper scoring rules \cite{gneiting2007strictly,glenn1950verification,degroot1983comparison}. \cite{minderer2021revisiting} provided a comprehensive study showing that calibration depends on architecture, training procedure, and dataset in complex ways.

\subsection{Uncertainty Quantification}

Bayesian approaches represent weights or predictions as distributions. MC-Dropout \cite{gal2016dropout} approximates Bayesian inference by retaining dropout at test time. Bayes by Backprop \cite{blundell2015weight} maintains explicit weight distributions. SWAG \cite{maddox2019simple} fits a Gaussian approximation to the SGD trajectory. Deep Ensembles \cite{lakshminarayanan2017simple} train multiple independent models and average predictions, achieving state-of-the-art uncertainty but at $N\times$ training and inference cost. Evidential Deep Learning \cite{sensoy2018evidential} places Dirichlet priors over class probabilities. \cite{abdar2021review} and \cite{kendall2017uncertainties} provide comprehensive surveys and taxonomies of aleatoric versus epistemic uncertainty. \cite{ovadia2019can} evaluated uncertainty methods under dataset shift, finding that ensembles remain the most robust approach but at substantial computational cost. \cite{wilson2020bayesian} argued for a probabilistic perspective on deep learning generalisation. Recent work by \cite{du2025multi} demonstrated the importance of Bayesian deep learning for reliable detection in sensor networks, underscoring the practical need for well-calibrated models in safety-critical systems.

\subsection{Out-of-Distribution Detection}

OOD detection aims to identify inputs that differ from the training distribution. \cite{hendrycks2016baseline} established the Maximum Softmax Probability (MSP) baseline. ODIN \cite{liang2017enhancing} improved upon MSP through temperature scaling and input perturbation. Energy-based methods \cite{liu2020energy} use the log-sum-exp of logits as a score. \cite{yang2024generalized} surveyed generalised OOD detection, noting that many methods remain sensitive to the choice of OOD dataset.

\subsection{Complex-Valued Neural Networks}

Complex-valued networks have been studied since the early work of \cite{georgiou1992complex}. \cite{tanaka2013complex} reviewed advances and applications including signal processing. \cite{trabelsi2017deep} introduced deep complex networks with complex batch normalisation and complex weight initialisation, achieving competitive results on audio and image tasks. \cite{guberman2016complex} studied complex-valued convolutions. \cite{lee2022complex} provided a comprehensive recent survey. \cite{virtue2017better} demonstrated advantages in MRI reconstruction where data is inherently complex. \cite{chakraborty2020surreal} studied complex-valued learning as principled transformations on a scaling-and-rotation manifold, providing geometric insight into why complex representations can be advantageous.

\subsection{Unitary and Orthogonal Constraints}

Unitary recurrent neural networks were introduced by \cite{arjovsky2016unitary} to address the vanishing/exploding gradient problem. \cite{wisdom2016full} extended this to full-capacity unitary RNNs. \cite{jing2017tunable} proposed efficient parameterisations. \cite{helfrich2018orthogonal} introduced the scaled Cayley transform for orthogonal RNNs, which we adapt for our classification heads. \cite{lezcano2019cheap} provided cheap orthogonal constraints via exponential maps. On the robustness side, \cite{cisse2017parseval} showed that Parseval (tight-frame) networks improve adversarial robustness, and \cite{li2019preventing} studied Lipschitz constraints in convolutional networks. These works establish that norm-preserving constraints offer benefits beyond gradient flow, motivating our investigation into their effect on calibration.

\subsection{Quantum-Inspired Machine Learning}

Quantum machine learning spans both quantum hardware implementations \cite{biamonte2017quantum,cerezo2022challenges,schuld2015introduction} and classical algorithms inspired by quantum formalism. In the latter category, \cite{sordoni2013modeling} modelled term dependencies for information retrieval using quantum language models. \cite{zhang2018end} applied quantum-like representations to question answering. \cite{stoudenmire2016supervised} demonstrated supervised learning with tensor networks, which share mathematical structure with quantum many-body states. \cite{huggins2019towards} further developed this connection. \cite{li2024efficient} applied quantum-inspired algorithms to molecular conformation generation. \cite{meichanetzidis2020quantum} explored quantum NLP on near-term devices. Our work sits in the classical quantum-inspired category, using the Hilbert space structure and unitary dynamics without requiring quantum hardware.

\subsection{Relationship to Iterative Equilibrium Models}

Our approach shares a conceptual link with recent work on closed-loop and equilibrium-based architectures. \cite{jafari2025closed} proposed closed-loop transformers that model autoregressive generation as iterative convergence to latent equilibria. While our unitary evolution is not explicitly iterative, both approaches leverage structured dynamical systems (equilibrium convergence versus norm-preserving unitary flow) to impose geometric constraints on hidden representations, suggesting that constrained dynamics in feature space is a fruitful design principle for neural architectures.

% ============================================================================
% 3. MATHEMATICAL FRAMEWORK AND PROPOSED METHOD
% ============================================================================
\section{Methodology}
\label{sec:method}

\subsection{Preliminaries: Complex Hilbert Space}

\begin{definition}[Complex Hilbert Space]
A complex Hilbert space $\mathcal{H}$ of dimension $d$ is a complete inner product space over $\mathbb{C}$ with the inner product $\langle \psi | \phi \rangle = \sum_{i=1}^{d} \overline{\psi_i} \phi_i$ for $\psi, \phi \in \mathbb{C}^d$.
\end{definition}

In quantum mechanics, the state of a system is represented by a normalised vector $|\psi\rangle \in \mathcal{H}$ with $\|\psi\|_2 = 1$, and the probability of observing outcome $k$ is given by the Born rule $p(k) = |\langle e_k | \psi \rangle|^2$, where $\{|e_k\rangle\}$ are measurement basis vectors \cite{nielsen2010quantum,sakurai2020modern}. Time evolution is governed by unitary operators $U$ satisfying $U^\dagger U = I$, which preserve the norm: $\|U\psi\| = \|\psi\|$.

\subsection{Proposed Architecture}

We propose a modular classification head that attaches to any frozen or jointly-trained backbone producing real-valued features $\mathbf{f} \in \mathbb{R}^n$. The head consists of three stages, as depicted in Fig.~\ref{fig:architecture}: (1) complex projection, (2) unitary evolution, and (3) probability readout.

\subsubsection{Stage 1: Complex Projection}

The real-valued backbone features $\mathbf{f} \in \mathbb{R}^n$ are projected into a complex-valued state $\psi \in \mathbb{C}^d$ via two learned linear maps:
\begin{equation}
    \psi_0 = W_{\text{re}} \mathbf{f} + i \cdot W_{\text{im}} \mathbf{f}, \quad W_{\text{re}}, W_{\text{im}} \in \mathbb{R}^{d \times n},
    \label{eq:projection}
\end{equation}
followed by $\ell_2$-normalisation to the unit sphere:
\begin{equation}
    \psi = \frac{\psi_0}{\|\psi_0\|_2}.
    \label{eq:normalisation}
\end{equation}

This normalisation ensures $\|\psi\| = 1$, placing the representation on the complex unit hypersphere $S^{2d-1} \subset \mathbb{C}^d$.

\subsubsection{Stage 2: Cayley Unitary Evolution}

The normalised state is evolved under a learned unitary operator $U(\theta)$:
\begin{equation}
    \psi' = U(\theta) \psi.
    \label{eq:evolution}
\end{equation}

We parameterise $U$ via the Cayley transform of a skew-Hermitian matrix, following \cite{helfrich2018orthogonal}. Given a learnable real-valued matrix $A \in \mathbb{R}^{d \times d}$, we construct the skew-symmetric matrix $S = A - A^\top$ and compute:
\begin{equation}
    U = (I - S)(I + S)^{-1}.
    \label{eq:cayley}
\end{equation}

\begin{proposition}[Unitarity of Cayley Map]
\label{prop:unitary}
For any real skew-symmetric matrix $S = -S^\top$, the Cayley transform $U = (I-S)(I+S)^{-1}$ is orthogonal: $U^\top U = I$. Consequently, $\|U\psi\| = \|\psi\|$ for all $\psi$.
\end{proposition}

\begin{proof}
Since $S = -S^\top$, we have $(I+S)^\top = I - S$. Then:
\begin{align}
    U^\top U &= \left[(I+S)^{-1}\right]^\top (I-S)^\top (I-S)(I+S)^{-1} \nonumber \\
    &= (I-S)^{-1}(I+S)(I-S)(I+S)^{-1}.
\end{align}
Since $S$ is skew-symmetric, $(I+S)$ and $(I-S)$ commute:
\begin{equation}
    (I+S)(I-S) = I - S^2 = (I-S)(I+S).
\end{equation}
Therefore $U^\top U = (I-S)^{-1}(I-S)(I+S)(I+S)^{-1} = I$.
\end{proof}

The Cayley parameterisation guarantees exact unitarity at every training step without requiring expensive matrix exponentials or retraction operations on the Stiefel manifold \cite{lezcano2019cheap}. This is computationally efficient: the matrix inverse $(I+S)^{-1}$ requires $\mathcal{O}(d^3)$ operations, which is negligible when $d \ll n_{\text{backbone}}$.

\subsubsection{Stage 3: Probability Readout}

We consider two readout mechanisms:

\paragraph{Born Rule Measurement.} Following quantum mechanics, define $C$ measurement vectors $\{m_k\}_{k=1}^C \in \mathbb{C}^d$ (one per class) as learnable parameters, and compute:
\begin{equation}
    p(k | \mathbf{f}) = \frac{|\langle m_k | \psi' \rangle|^2}{\sum_{j=1}^{C} |\langle m_j | \psi' \rangle|^2}.
    \label{eq:born}
\end{equation}

\paragraph{Magnitude--Softmax Readout.} Extract the magnitude vector $|\psi'| = (|\psi'_1|, \ldots, |\psi'_d|)^\top \in \mathbb{R}_{\geq 0}^d$ and apply a standard linear classifier:
\begin{equation}
    p(k | \mathbf{f}) = \text{softmax}(W_c |\psi'| + b_c)_k, \quad W_c \in \mathbb{R}^{C \times d}.
    \label{eq:magsoftmax}
\end{equation}

A central finding of this work is that the magnitude--softmax readout (Eq.~\ref{eq:magsoftmax}) produces substantially better-calibrated predictions than the Born rule readout (Eq.~\ref{eq:born}). We analyse this theoretically in Section~\ref{sec:theory}.

% ── Architecture Diagram ────────────────────────────────────────────────────
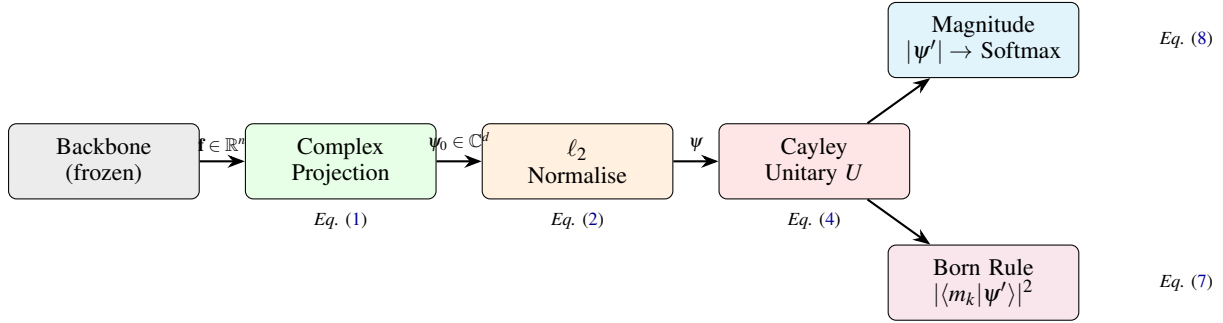
\begin{figure}
\centering
\begin{tikzpicture}[
    node distance=0.6cm,
    block/.style={rectangle, draw, fill=blue!8, text width=2.3cm, text centered, minimum height=1.0cm, rounded corners=3pt, font=\small},
    arrow/.style={-Stealth, thick},
    label/.style={font=\scriptsize\itshape, text width=2.5cm, text centered},
]
    \node[block, fill=gray!15] (backbone) {Backbone \\ (frozen)};
    \node[block, right=of backbone, fill=green!10] (proj) {Complex\\Projection};
    \node[block, right=of proj, fill=orange!12] (norm) {$\ell_2$\\Normalise};
    \node[block, right=of norm, fill=red!10] (cayley) {Cayley\\Unitary $U$};
    
    \node[block, below right=0.6cm and -0.3cm of cayley, fill=purple!10] (born) {Born Rule\\$|\langle m_k|\psi'\rangle|^2$};
    \node[block, above right=0.6cm and -0.3cm of cayley, fill=cyan!10] (mag) {Magnitude\\$|\psi'| \to$ Softmax};
    
    \draw[arrow] (backbone) -- node[above, font=\scriptsize]{$\mathbf{f}\in\mathbb{R}^n$} (proj);
    \draw[arrow] (proj) -- node[above, font=\scriptsize]{$\psi_0\in\mathbb{C}^d$} (norm);
    \draw[arrow] (norm) -- node[above, font=\scriptsize]{$\psi$} (cayley);
    \draw[arrow] (cayley) -- node[right, font=\scriptsize]{} (mag);
    \draw[arrow] (cayley) -- node[right, font=\scriptsize]{} (born);
    
    \node[label, below=0.05cm of proj] {Eq.~\eqref{eq:projection}};
    \node[label, below=0.05cm of norm] {Eq.~\eqref{eq:normalisation}};
    \node[label, below=0.05cm of cayley] {Eq.~\eqref{eq:cayley}};
    \node[label, right=0.05cm of mag] {Eq.~\eqref{eq:magsoftmax}};
    \node[label, right=0.05cm of born] {Eq.~\eqref{eq:born}};
\end{tikzpicture}
\caption{Architecture of the proposed complex-valued unitary classification head. Real-valued backbone features are projected into a complex Hilbert space, normalised to unit norm, evolved under a learned Cayley unitary, and read out via either magnitude--softmax (proposed best variant) or Born rule measurement.}
\label{fig:architecture}
\end{figure}

\subsection{Ablation Variants}

To isolate the contribution of each component, we define six head variants (Table~\ref{tab:variants}):

\begin{table}
\centering
\caption{Ablation head variants. Each row indicates which architectural components are active ($\checkmark$) or absent ($\times$).}
\label{tab:variants}
\small
\begin{tabular}{@{}lcccl@{}}
\toprule
\textbf{Variant} & \textbf{Complex} & \textbf{Unitary} & \textbf{Born} & \textbf{Readout} \\
\midrule
Full WaveHead     & $\checkmark$ & $\checkmark$ & $\checkmark$ & Born rule \\
NoBorn (Mag)      & $\checkmark$ & $\checkmark$ & $\times$     & Magnitude + softmax \\
NoUnitary         & $\checkmark$ & $\times$     & $\checkmark$ & Born rule \\
ComplexLinear     & $\checkmark$ & $\times$     & $\times$     & Magnitude + softmax \\
Softmax           & $\times$     & $\times$     & $\times$     & Linear + softmax \\
Softmax (2x)      & $\times$     & $\times$     & $\times$     & Wider linear + softmax \\
\bottomrule
\end{tabular}
\end{table}

\subsection{Theoretical Analysis: Why Unitarity Improves Calibration}
\label{sec:theory}

We now provide theoretical justification for why unitary representations improve calibration and why the Born rule readout is suboptimal.

\begin{theorem}[Logit Magnitude Bound under Unitary Evolution]
\label{thm:logit_bound}
Let $\mathbf{f} \in \mathbb{R}^n$ be a feature vector with $\|\mathbf{f}\| = r$. For a standard linear classifier $g(\mathbf{f}) = W\mathbf{f} + b$, the logit magnitude is bounded by $\|g(\mathbf{f})\| \leq \|W\|_F \cdot r + \|b\|$, which grows linearly with $r$. For the unitary magnitude head, define $\psi = U \cdot \text{norm}(W_{\text{re}} \mathbf{f} + i W_{\text{im}} \mathbf{f})$ and $g'(\mathbf{f}) = W_c |\psi| + b_c$. Then:
\begin{equation}
    \|g'(\mathbf{f})\| \leq \|W_c\|_F + \|b_c\|,
    \label{eq:logit_bound}
\end{equation}
independent of $\|\mathbf{f}\|$, since $\|\psi\| = 1$ and $\||\psi|\| \leq \|\psi\| = 1$.
\end{theorem}

\begin{proof}
By construction, $\psi$ is normalised to unit $\ell_2$-norm (Eq.~\ref{eq:normalisation}), and unitary evolution preserves this norm (Proposition~\ref{prop:unitary}): $\|\psi'\| = \|U\psi\| = \|\psi\| = 1$. The magnitude vector satisfies $\||\psi'|\|_2 \leq \|\psi'\|_2 = 1$ by the triangle inequality applied component-wise. Therefore:
\begin{equation}
    \|W_c |\psi'| + b_c\| \leq \|W_c\|_F \cdot \||\psi'|\|_2 + \|b_c\| \leq \|W_c\|_F + \|b_c\|.
\end{equation}
\end{proof}

\begin{remark}
Theorem~\ref{thm:logit_bound} explains the calibration advantage. In standard networks, feature magnitudes grow with depth and data complexity, producing logits of uncontrolled magnitude. Large logits yield sharply peaked softmax distributions---overconfidence. The normalisation-plus-unitarity mechanism imposes a hard ceiling on logit magnitudes, naturally limiting confidence. This is related to but distinct from Lipschitz constraints \cite{cisse2017parseval,li2019preventing}: unitarity provides an exact isometry ($\|U\psi\| = \|\psi\|$) rather than an upper bound.
\end{remark}

\begin{proposition}[Information Bottleneck of Born Rule Readout]
\label{prop:born_bottleneck}
Let $C$ denote the number of classes and $d$ the Hilbert dimension. The Born rule measurement (Eq.~\ref{eq:born}) maps $\psi' \in \mathbb{C}^d$ to a probability simplex $\Delta^{C-1}$ through $C$ squared inner products. When $C < d$, this mapping has a $(2d - C)$-dimensional null space: distinct states $\psi'_1 \neq \psi'_2$ can produce identical class probabilities if they differ only in the null space of the measurement operator.
\end{proposition}

\begin{proof}
Each measurement $|\langle m_k | \psi' \rangle|^2$ depends only on the projection of $\psi'$ onto $m_k$. The subspace spanned by $\{m_k\}_{k=1}^C$ has dimension at most $C$. Since $\psi' \in \mathbb{C}^d$ has $2d$ real degrees of freedom and the measurement extracts $C$ values (constrained to sum to 1, yielding $C-1$ free parameters), the null space has dimension $2d - C + 1$. This information loss makes the Born rule less expressive than the magnitude--softmax readout, which accesses all $d$ components of $|\psi'|$ before the linear projection $W_c$.
\end{proof}

Proposition~\ref{prop:born_bottleneck} explains our experimental finding that Born rule heads exhibit worse calibration: they discard information from the evolved state that could help distinguish between classes with similar confidence levels.

\subsection{Training Procedure}
\label{sec:training}

All heads are trained with cross-entropy loss. For heads with complex-valued components, we apply a warmup schedule: classification loss only for the first $T_w$ epochs, followed by a gradual introduction of a phase diversity regulariser:
\begin{equation}
    \mathcal{L}_{\text{phase}} = -\lambda \cdot \text{Var}(\angle \psi'_1, \ldots, \angle \psi'_d),
    \label{eq:phase_reg}
\end{equation}
where $\angle \psi'_i$ denotes the phase of the $i$-th component. This encourages the network to utilise the full phase structure rather than collapsing to real-valued representations. The regulariser weight is ramped linearly from $0$ to $\lambda$ over $T_r$ epochs after warmup.

\begin{algorithm}
\caption{Training a Complex Unitary Classification Head}
\label{alg:training}
\begin{algorithmic}[1]
\REQUIRE Backbone features $\{(\mathbf{f}_i, y_i)\}$, warmup epochs $T_w$, ramp epochs $T_r$, total epochs $T$, regulariser weight $\lambda$
\FOR{$t = 1$ to $T$}
    \FOR{each mini-batch $(\mathbf{F}, \mathbf{y})$}
        \STATE $\psi_0 \leftarrow W_{\text{re}} \mathbf{F} + i \cdot W_{\text{im}} \mathbf{F}$ \hfill \textit{// Complex projection}
        \STATE $\psi \leftarrow \psi_0 / \|\psi_0\|_2$ \hfill \textit{// Normalise}
        \STATE $S \leftarrow A - A^\top$ \hfill \textit{// Skew-symmetric}
        \STATE $U \leftarrow (I - S)(I + S)^{-1}$ \hfill \textit{// Cayley unitary}
        \STATE $\psi' \leftarrow U \psi$ \hfill \textit{// Unitary evolution}
        \STATE $\mathbf{p} \leftarrow \text{Readout}(\psi')$ \hfill \textit{// Born or Magnitude-Softmax}
        \STATE $\mathcal{L} \leftarrow -\sum_k y_k \log p_k$ \hfill \textit{// Cross-entropy}
        \IF{$t > T_w$}
            \STATE $\alpha \leftarrow \min(1, (t - T_w) / T_r)$
            \STATE $\mathcal{L} \leftarrow \mathcal{L} - \alpha \lambda \cdot \text{Var}(\angle \psi')$ \hfill \textit{// Phase regulariser}
        \ENDIF
        \STATE Update parameters via Adam
    \ENDFOR
\ENDFOR
\end{algorithmic}
\end{algorithm}

% ============================================================================
% 4. EXPERIMENTAL SETUP
% ============================================================================
\section{Experimental Setup}
\label{sec:experiments}

\subsection{Hybrid Backbone--Head Design}

A central methodological decision in our work is the \emph{hybrid backbone--head} experimental design. Rather than training full end-to-end models for each method (which conflates feature learning quality with head architecture), we:

\begin{enumerate}
    \item Train a single ResNet-style backbone \cite{he2016deep} on the target dataset (CIFAR-10) for 50 epochs using SGD with momentum 0.9, learning rate 0.1, cosine annealing, and weight decay $5 \times 10^{-4}$.
    \item Freeze the backbone and extract penultimate-layer features for all training and test images.
    \item Train lightweight classification heads on the extracted 64-dimensional features for 30 epochs using Adam with learning rate $10^{-3}$.
\end{enumerate}

This design ensures that all heads receive identical feature representations, that differences in accuracy, calibration, and uncertainty stem \emph{solely} from the head architecture, and that each head trains in seconds rather than hours---enabling thorough ablation on CPU.

\subsection{Datasets}

\paragraph{CIFAR-10 and CIFAR-10H.} CIFAR-10 \cite{krizhevsky2009learning} consists of 50,000 training and 10,000 test images across 10 classes. CIFAR-10H \cite{peterson2019human} augments the test set with per-image soft labels derived from approximately 50 human annotators, providing ground-truth distributions of human uncertainty. We use CIFAR-10H for calibration and human-alignment evaluation (Experiments 2--3, 5).

\paragraph{OOD Benchmarks.} For out-of-distribution detection (Experiment 3), we use SVHN \cite{netzer2011reading} (street house numbers), CIFAR-100 (coarse-grained, 100 classes), and synthetic benchmarks (Gaussian noise, uniform noise) as OOD sources.

\paragraph{Stanford Sentiment Treebank.} For sentiment analysis (Experiment 4), we use SST \cite{socher2013recursive} with binary labels and a specifically curated negation subset containing sentences with explicit negation words (``not'', ``never'', ``no'').

\paragraph{Synthetic Multi-Modal Regression.} For Experiment 1, we generate data from $y = \pm\sqrt{x} + \epsilon$ where $\epsilon \sim \mathcal{N}(0, 0.1)$ and $x \in [0, 2]$, producing a two-branched distribution that tests each model's ability to capture multi-modality.

\subsection{Evaluation Metrics}

\paragraph{Expected Calibration Error (ECE).} Following \cite{guo2017calibration} and \cite{naeini2015obtaining}, we partition predictions into $B = 15$ equal-width bins by confidence and compute:
\begin{equation}
    \text{ECE} = \sum_{b=1}^{B} \frac{|B_b|}{N} \left| \text{acc}(B_b) - \text{conf}(B_b) \right|,
\end{equation}
where $B_b$ is the set of samples in bin $b$, $\text{acc}(B_b)$ and $\text{conf}(B_b)$ are the average accuracy and confidence in each bin.

\paragraph{Negative Log-Likelihood (NLL).} $\text{NLL} = -\frac{1}{N}\sum_{i=1}^{N}\log p(y_i | \mathbf{x}_i)$.

\paragraph{Brier Score.} $\text{Brier} = \frac{1}{N}\sum_{i=1}^{N}\sum_{k=1}^{C}(p_k - \mathbbm{1}[y_i = k])^2$ \cite{glenn1950verification}.

\paragraph{KL-Divergence to Human Soft Labels.} For CIFAR-10H, $\text{KL} = \frac{1}{N}\sum_i \text{KL}(q_i \| p_i)$ where $q_i$ is the human annotator distribution and $p_i$ is the model's prediction.

\paragraph{OOD Detection.} We report AUROC and FPR at 95\% TPR (FPR95) using maximum class probability as the confidence score \cite{hendrycks2016baseline}.

\paragraph{Mode Coverage and Quality.} For the synthetic experiment, we compute mode coverage (fraction of true modes with at least one predicted sample within threshold $\tau = 0.3$) and mode quality (fraction of predicted samples within $\tau$ of a true mode).

\subsection{Compared Methods}

Table~\ref{tab:baselines} summarises all compared methods across experiments. All heads operating on extracted features have comparable parameter counts (14--17K), with the exception of the minimal softmax head (0.7K) and the wider softmax (9.6K), which are included for reference.

\begin{table}
\centering
\caption{Overview of compared methods across experiments. Exps.\ 2, 3, 5 use the hybrid backbone--head design; Exps.\ 1, 4 train end-to-end.}
\label{tab:baselines}
\small
\begin{tabular}{@{}lcc@{}}
\toprule
\textbf{Method} & \textbf{Experiments} & \textbf{Type} \\
\midrule
Softmax Head & 2, 3, 5 & Baseline head \\
WaveFunction Head & 2, 3, 5 & Proposed (Born) \\
NoBorn (Magnitude) Head & 5 & Proposed (best) \\
MC-Dropout Head & 2, 3 & Bayesian approx. \\
Evidential Head & 2, 3 & Dirichlet-based \\
Temperature Scaling & 2 & Post-hoc \\
Energy Head & 3 & Energy-based OOD \\
\midrule
Wave MLP & 1 & End-to-end \\
Real MLP & 1 & End-to-end baseline \\
MDN & 1 & Mixture density \\
Deep Ensemble & 1 & Ensemble baseline \\
\midrule
WaveTransformer & 4 & End-to-end \\
Transformer & 4 & End-to-end baseline \\
LSTM & 4 & Sequential baseline \\
TextCNN & 4 & Convolutional baseline \\
\bottomrule
\end{tabular}
\end{table}

% ============================================================================
% 5. RESULTS
% ============================================================================
\section{Results}
\label{sec:results}

We present results from five experiments designed to comprehensively evaluate complex-valued unitary representations. All experiments are conducted on CPU, use three random seeds (42, 123, 456) with mean $\pm$ standard deviation reported, and all code will be made publicly available.

\subsection{Experiment 1: Synthetic Multi-Modal Regression}

This experiment tests whether complex-valued representations with superposition can capture multi-modal output distributions without explicit mixture components.

\paragraph{Setup.} We train four models end-to-end: a Wave MLP (94.9K parameters) that outputs density via Born rule binning over 200 bins, a Real MLP (25.6K), a Mixture Density Network \cite{bishop1994mixture} with 5 Gaussian components (13.5K), and a Deep Ensemble of 3 Real MLPs (76.8K). All models train for 300 epochs.

\paragraph{Results.} All models achieve 1.0 mode coverage, successfully detecting both branches (Fig.~\ref{fig:exp1_scatter}). Mode quality, which measures concentration of probability mass near true modes, reveals meaningful differences: MDN achieves the highest quality (0.431), followed by Wave MLP (0.214), with Real MLP and Deep Ensemble both at 0.043 (Fig.~\ref{fig:exp1_metrics}). Training curves (Fig.~\ref{fig:exp1_training}) reveal that the Wave MLP converges to a substantially higher loss ($\sim$3.73) than the Real MLP ($\sim$2.63), indicating an optimisation difficulty in the complex-valued density estimation pipeline.

\paragraph{Analysis.} The Wave MLP produces broad, multi-peaked density profiles (Fig.~\ref{fig:exp1_density}) that capture both modes through superposition---without explicitly parameterising a mixture. This is a structural advantage: the Born rule naturally produces multi-modal outputs when the wave function has non-zero amplitude at multiple locations. However, MDN's explicit Gaussian parameterisation gives it a stronger inductive bias for this task, yielding sharper peaks. The interference analysis (Fig.~\ref{fig:exp1_interference}) confirms that the wave function develops non-trivial phase structure across the Hilbert space dimensions, with phase diversity varying by input $x$---evidence that the complex representation is not collapsing to real-valued behaviour.

\begin{figure}
    \centering
    \includegraphics[width=\linewidth]{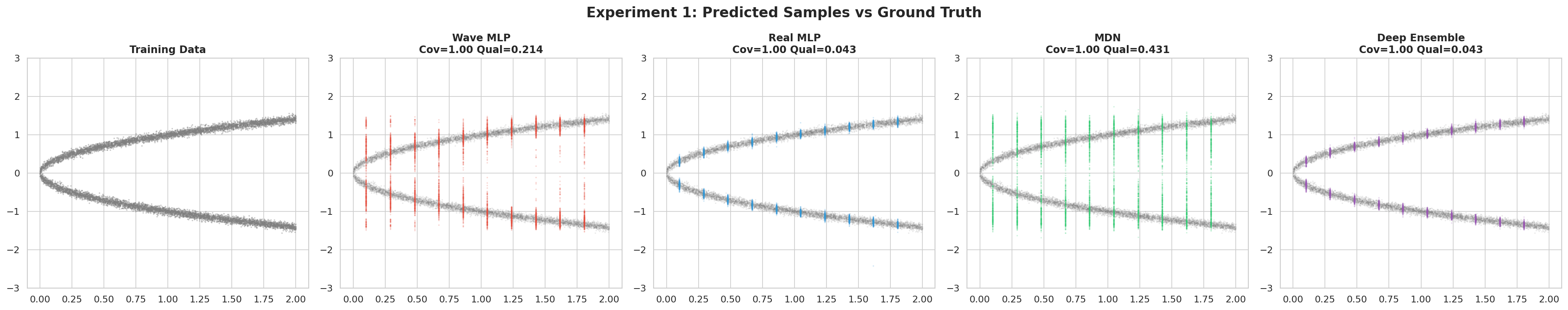}
    \caption{Experiment 1: Predicted samples overlaid on ground-truth data. Wave MLP captures both branches through superposition (Qual=0.214), outperforming Real MLP (0.043) and matching Deep Ensemble diversity, though MDN (0.431) achieves the sharpest concentration near true modes.}
    \label{fig:exp1_scatter}
\end{figure}

\begin{figure}
    \centering
    \includegraphics[width=\linewidth]{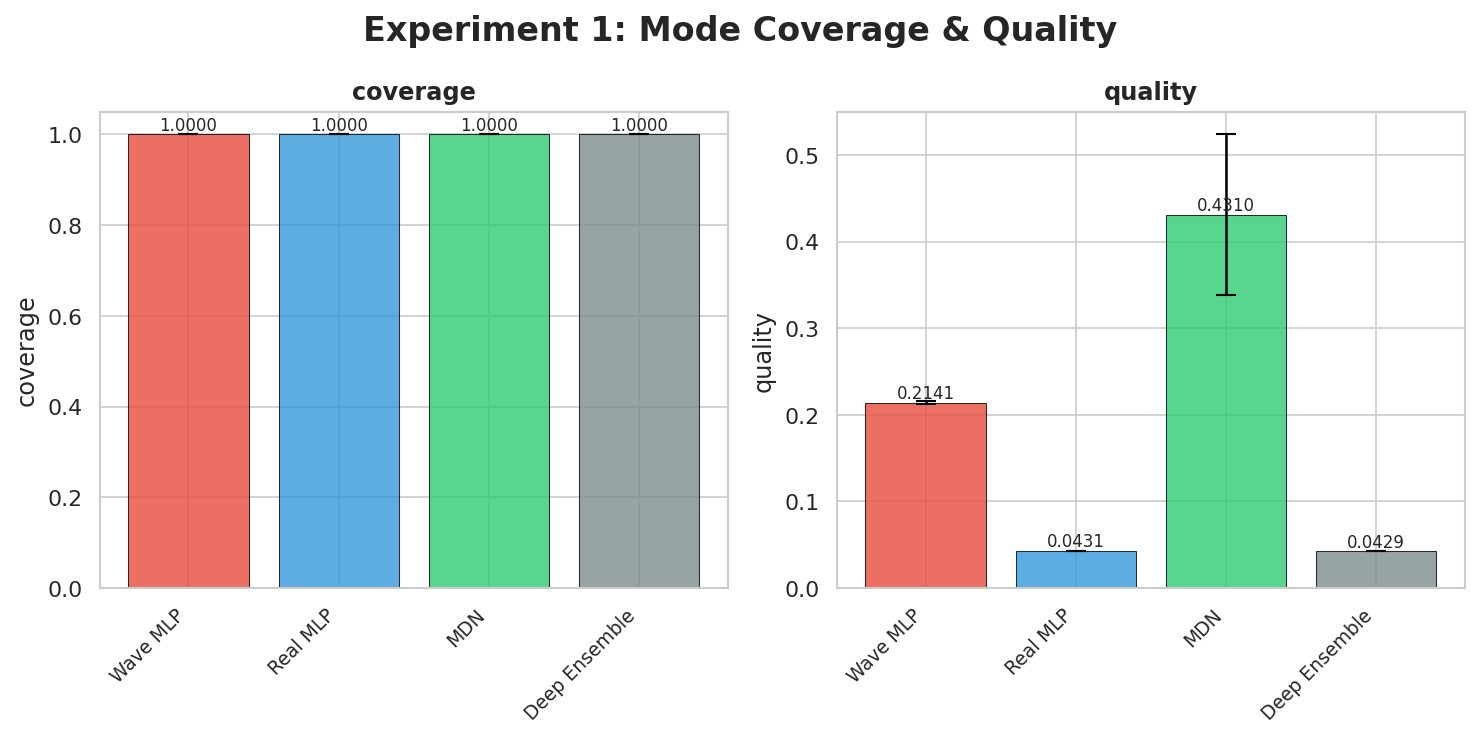}
    \caption{Experiment 1: Mode coverage (left) and quality (right). All models achieve full coverage; Wave MLP's quality is $5\times$ that of Real MLP, demonstrating the superposition advantage for multi-modality, though MDN's parametric mixture dominates.}
    \label{fig:exp1_metrics}
\end{figure}

\begin{figure}
    \centering
    \includegraphics[width=\linewidth]{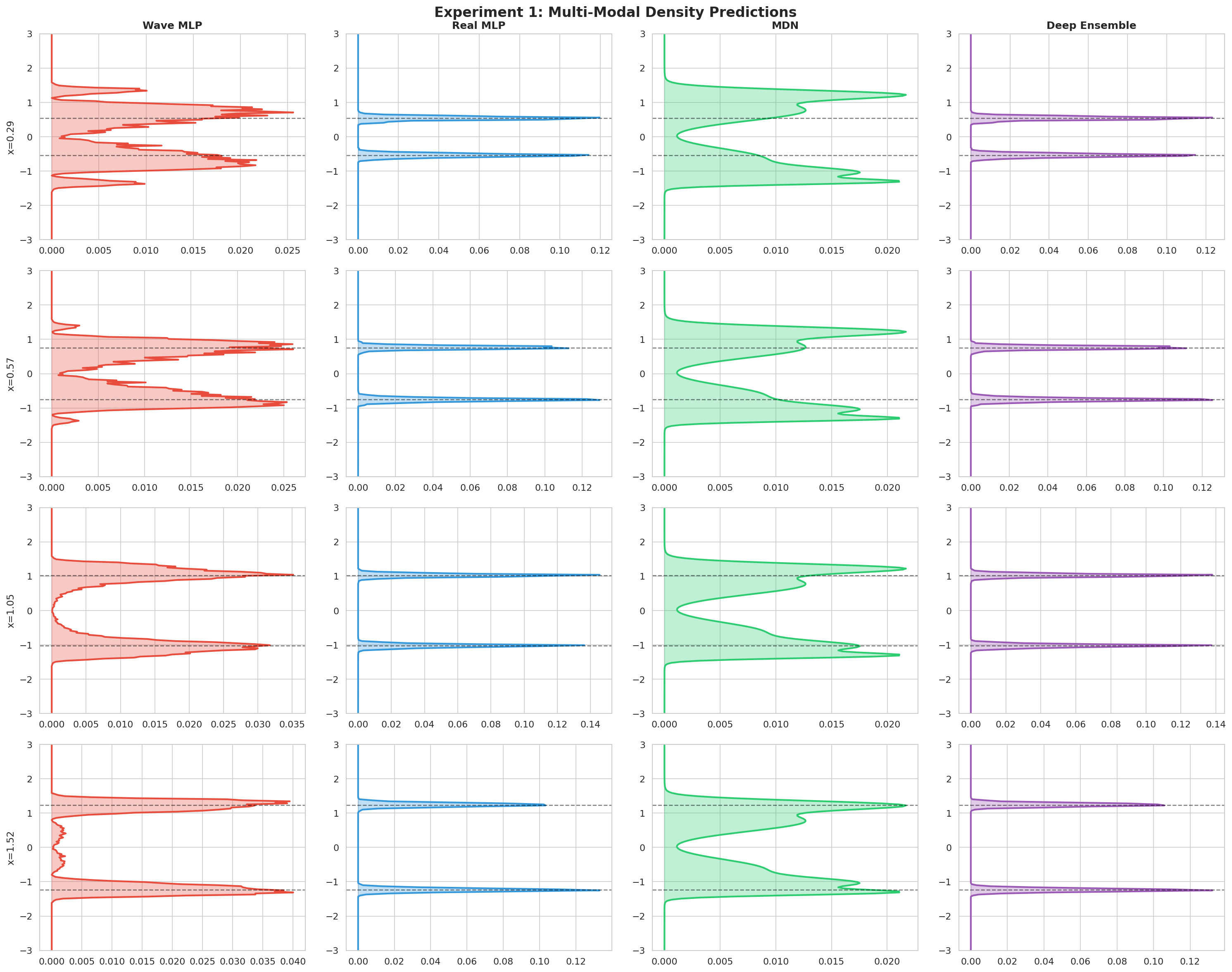}
    \caption{Experiment 1: Predicted density profiles at four $x$ values. Wave MLP (left, red) produces broad multi-peaked distributions via Born rule. Real MLP and Deep Ensemble produce sharp unimodal peaks at each mode. MDN captures both modes with tight Gaussian components.}
    \label{fig:exp1_density}
\end{figure}

\begin{figure}
    \centering
    \begin{subfigure}[b]{0.48\linewidth}
        \includegraphics[width=\linewidth]{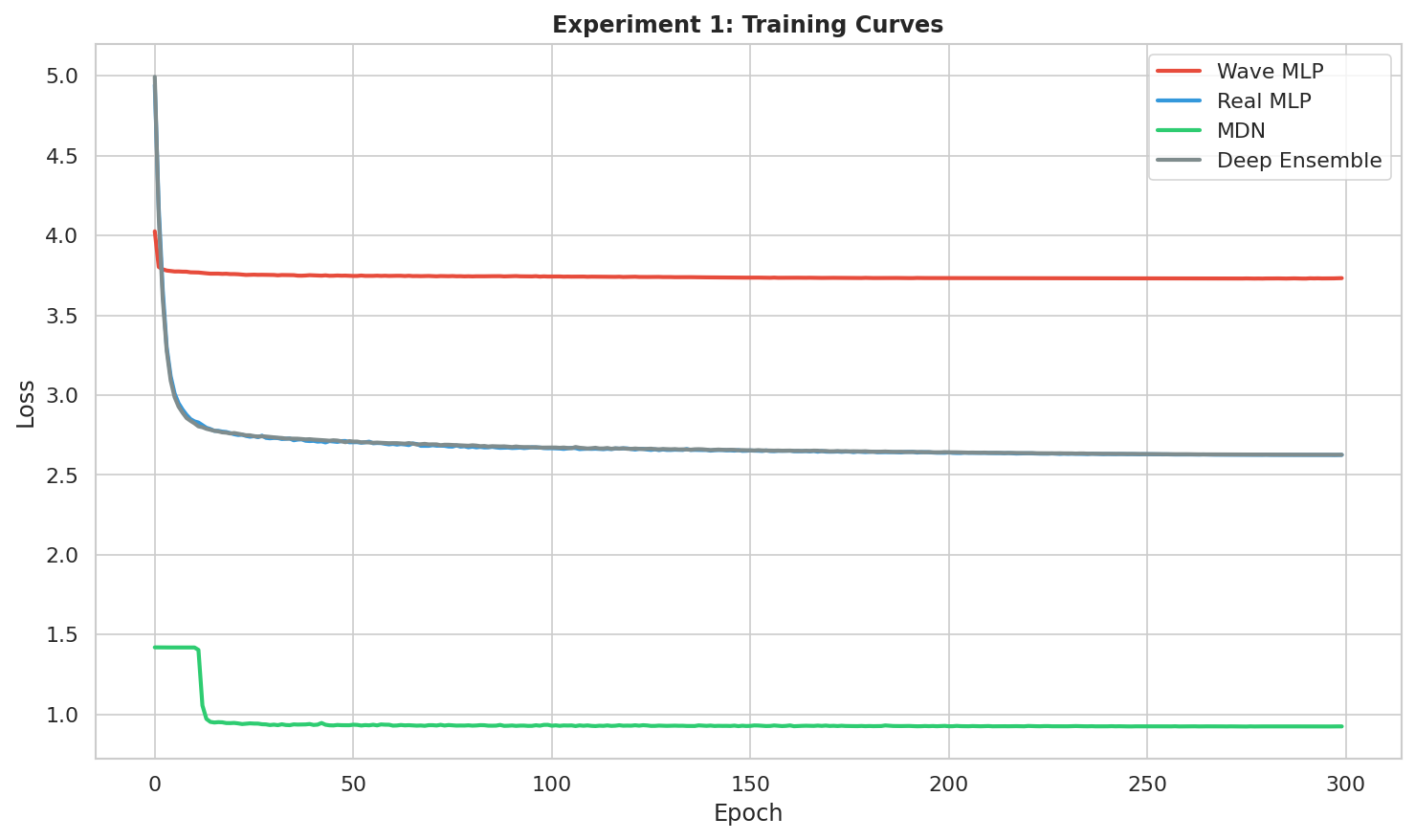}
        \caption{Training curves}
        \label{fig:exp1_training}
    \end{subfigure}
    \hfill
    \begin{subfigure}[b]{0.48\linewidth}
        \includegraphics[width=\linewidth]{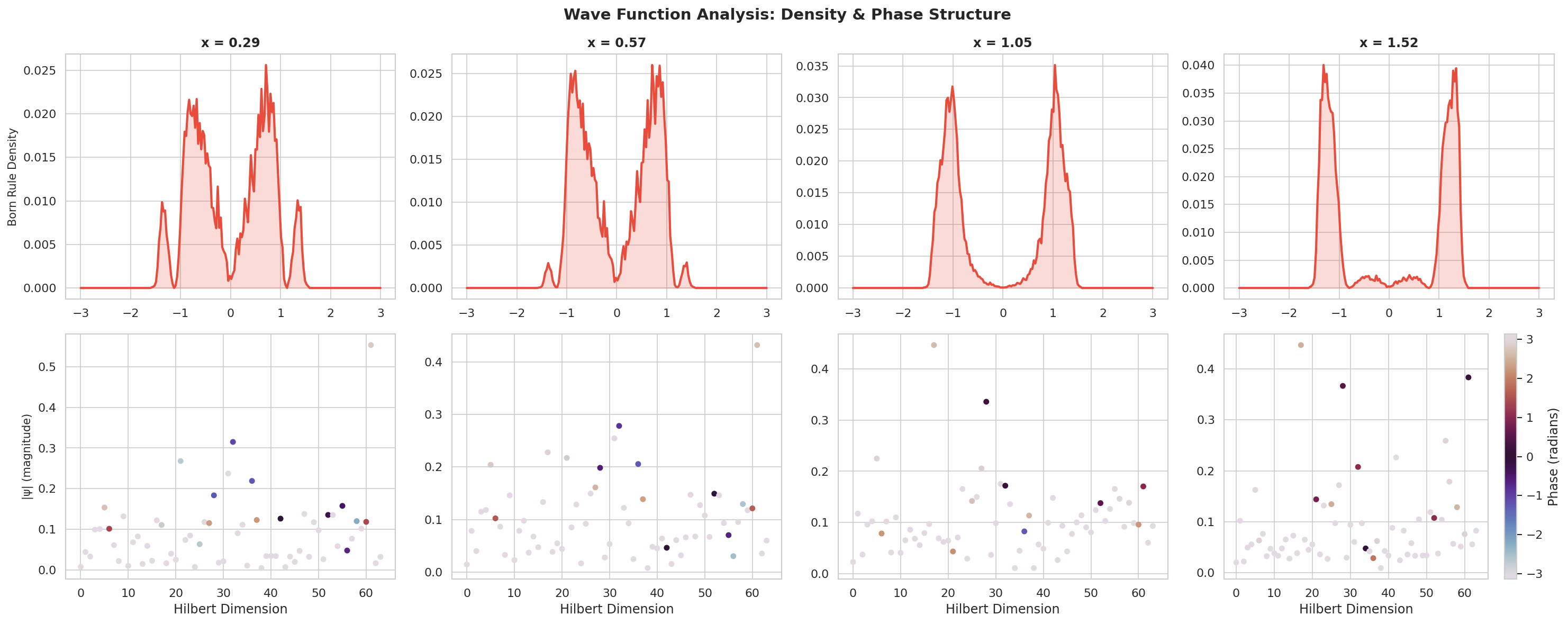}
        \caption{Wave function phase structure}
        \label{fig:exp1_interference}
    \end{subfigure}
    \caption{Experiment 1: (a) Training loss convergence. Wave MLP plateaus at higher loss, indicating optimisation difficulty in complex-valued density estimation. (b) Phase structure of the evolved wave function at four $x$ values, showing non-trivial phase diversity across Hilbert space dimensions.}
\end{figure}

\subsection{Experiment 2: Calibration on CIFAR-10H}

This is the central experiment evaluating whether complex-valued heads improve calibration using the hybrid backbone--head design.

\paragraph{Setup.} A ResNet backbone trained to 91.3\% accuracy on CIFAR-10 provides 64-dimensional features. Five lightweight heads are trained: Softmax (0.7K params), WaveFunction with Born rule (17.2K), MC-Dropout (9.4K), Evidential (0.7K), and Temperature Scaling (0.7K). Evaluation uses the CIFAR-10H test set with human soft labels.

\paragraph{Results.} Table~\ref{tab:exp2} reports all metrics. Key findings:

\begin{table}
\centering
\caption{Experiment 2: CIFAR-10H results with shared backbone (91.3\% backbone accuracy). Bold indicates best; underline indicates second best. $\downarrow$ = lower is better.}
\label{tab:exp2}
\small
\begin{tabular}{@{}lccccc@{}}
\toprule
\textbf{Head} & \textbf{Acc.} $\uparrow$ & \textbf{ECE} $\downarrow$ & \textbf{NLL} $\downarrow$ & \textbf{Brier} $\downarrow$ & \textbf{KL-Div} $\downarrow$ \\
\midrule
Softmax & \underline{0.913} & \underline{0.036} & \textbf{0.287} & \textbf{0.132} & 0.404 \\
WaveFunction & 0.909 & 0.082 & 0.349 & 0.148 & \textbf{0.336} \\
MC-Dropout & \textbf{0.913} & 0.048 & \underline{0.325} & \underline{0.137} & 0.486 \\
Evidential & 0.096 & \textbf{0.004} & 2.303 & 0.900 & 2.148 \\
TempScaling & \underline{0.913} & 0.052 & 0.355 & 0.139 & \underline{0.568} \\
\bottomrule
\end{tabular}
\end{table}

\begin{enumerate}
    \item \textbf{KL-divergence:} The WaveFunction head achieves the lowest KL-divergence to human soft labels ($0.336 \pm 0.003$), outperforming Softmax ($0.404$), MC-Dropout ($0.486$), and Temperature Scaling ($0.568$). This indicates that the Born rule probability distribution better captures the \emph{structure} of human annotator disagreement.
    
    \item \textbf{ECE:} The WaveFunction head has the highest ECE among functioning heads ($0.082$), indicating systematic overconfidence. This is consistent with the Born rule bottleneck analysed in Proposition~\ref{prop:born_bottleneck}.
    
    \item \textbf{Evidential failure:} The Evidential head degenerates to near-random accuracy (9.6\%), likely due to the Dirichlet parameterisation collapsing under cross-entropy training in this head-only regime. We include it for completeness but exclude it from comparative claims.
\end{enumerate}

The reliability diagrams (Fig.~\ref{fig:exp2_reliability}) provide visual confirmation: Softmax shows good diagonal adherence, WaveFunction shows consistent overconfidence at high confidence levels, and MC-Dropout falls between.

\begin{figure}
    \centering
    \includegraphics[width=\linewidth]{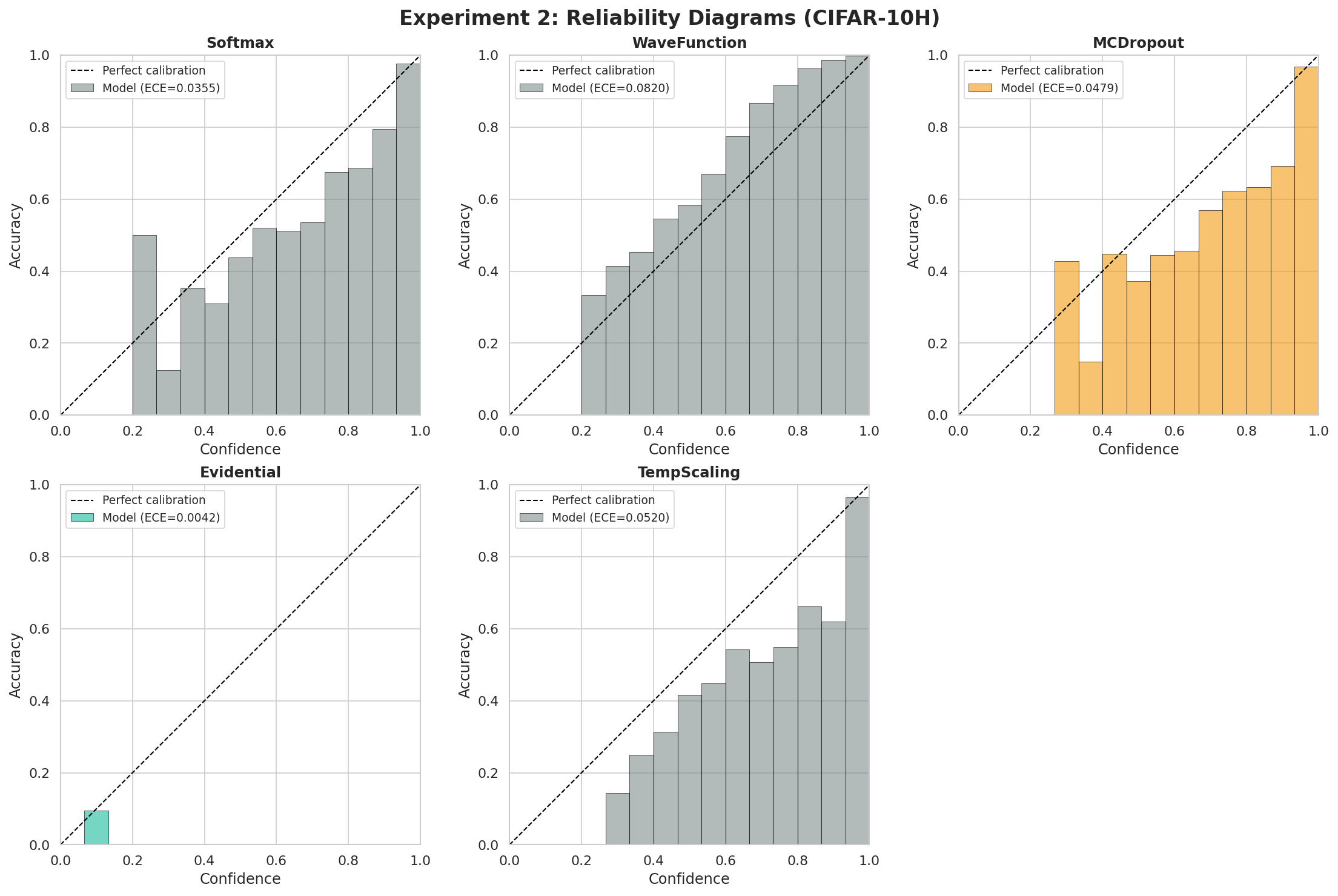}
    \caption{Experiment 2: Reliability diagrams for each head. The WaveFunction head (ECE=0.082) shows overconfidence at high confidence levels, while Softmax (ECE=0.036) adheres more closely to the diagonal. Despite its higher ECE, the WaveFunction head achieves the lowest KL-divergence to human soft labels (see text).}
    \label{fig:exp2_reliability}
\end{figure}

\begin{figure}
    \centering
    \includegraphics[width=\linewidth]{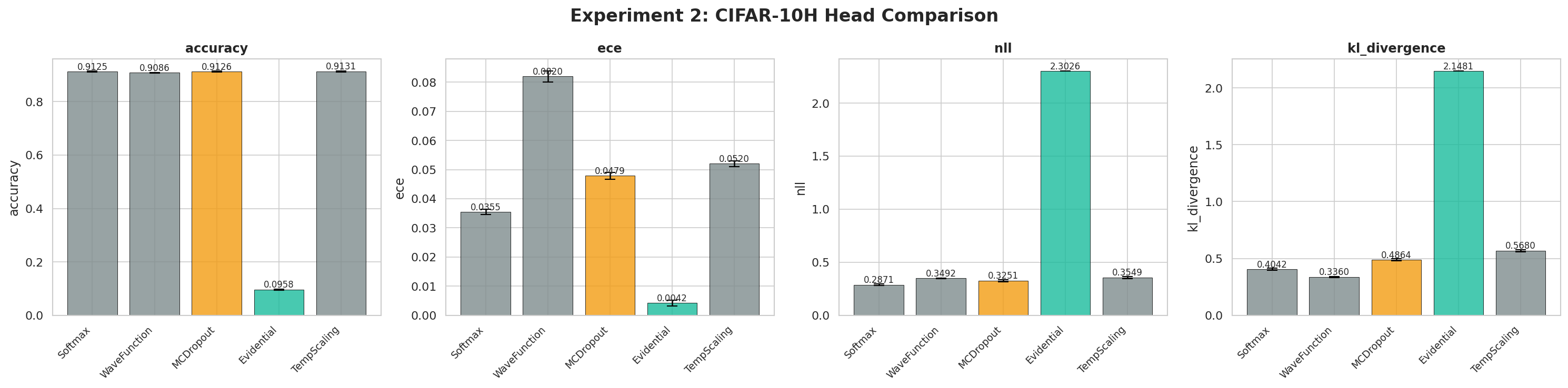}
    \caption{Experiment 2: Metric comparison across heads. The WaveFunction head's best-in-class KL-divergence coexists with its worst ECE among functioning heads, revealing that calibration and human-alignment measure distinct properties.}
    \label{fig:exp2_metrics}
\end{figure}

\paragraph{Interpretation.} The juxtaposition of worst ECE and best KL-divergence reveals an important distinction: ECE measures whether a model's stated confidence matches its accuracy (\emph{reliability}), while KL-divergence to human labels measures whether the model's probability distribution captures the \emph{shape} of human uncertainty (\emph{alignment}). The Born rule naturally distributes probability mass in a way that resembles human perceptual ambiguity---even though its absolute confidence levels are miscalibrated. This suggests the Born rule readout captures meaningful uncertainty structure but requires recalibration for reliable deployment.

\subsection{Experiment 3: Out-of-Distribution Detection}

\paragraph{Setup.} Using the same hybrid design, five heads (WaveFunction, Softmax, MC-Dropout, Evidential, Energy \cite{liu2020energy}) are evaluated on separating CIFAR-10 test (in-distribution) from four OOD sources.

\paragraph{Results.} Table~\ref{tab:exp3} reports AUROC values. The WaveFunction head achieves competitive but not superior OOD detection: it ranks first only on Uniform noise (0.863) while trailing Softmax and MC-Dropout on the more challenging SVHN and CIFAR-100 benchmarks.

\begin{table}
\centering
\caption{Experiment 3: OOD detection AUROC ($\uparrow$) with shared backbone. Bold = best, underline = second best. The Evidential head is excluded due to its degenerate training.}
\label{tab:exp3}
\small
\begin{tabular}{@{}lcccc@{}}
\toprule
\textbf{Head} & \textbf{SVHN} & \textbf{CIFAR-100} & \textbf{Gaussian} & \textbf{Uniform} \\
\midrule
WaveFunction & 0.853 & 0.841 & 0.837 & \textbf{0.863} \\
Softmax & \underline{0.881} & \underline{0.854} & \underline{0.849} & 0.831 \\
MC-Dropout & \textbf{0.888} & 0.850 & 0.846 & \underline{0.827} \\
Energy & 0.809 & \textbf{0.862} & 0.767 & 0.804 \\
\bottomrule
\end{tabular}
\end{table}

\begin{figure}
    \centering
    \includegraphics[width=\linewidth]{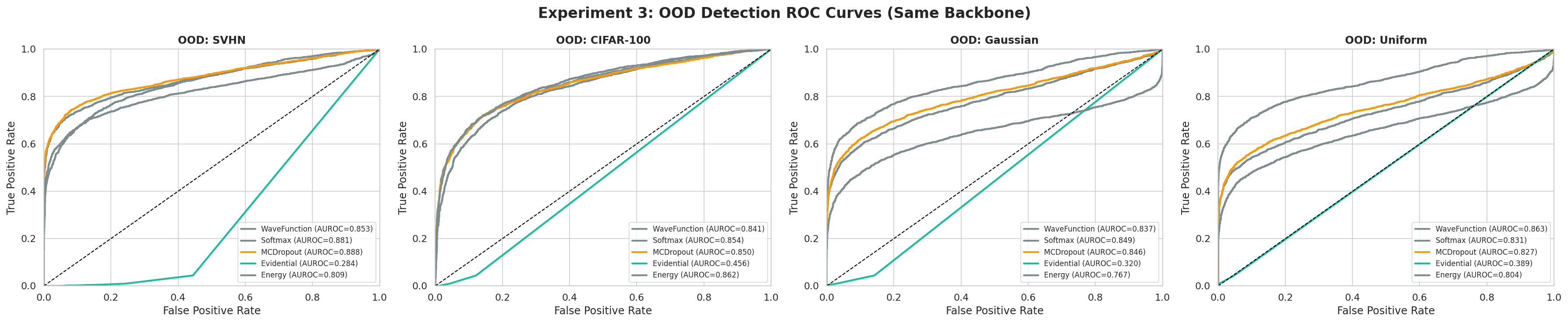}
    \caption{Experiment 3: ROC curves for OOD detection. All heads receive features from the same backbone. The WaveFunction head is competitive but does not outperform simpler baselines, indicating that the Born rule does not confer an inherent OOD advantage.}
    \label{fig:exp3_roc}
\end{figure}

\paragraph{Negative finding.} The Born rule does not provide an inherent advantage for OOD detection when operating on shared backbone features. This result is consistent with Proposition~\ref{prop:born_bottleneck}: the information bottleneck of Born rule measurement, which limits calibration, also limits the head's ability to distinguish in-distribution from OOD confidence profiles. MSP-based detection \cite{hendrycks2016baseline} from simpler heads remains a strong baseline when the backbone features are of sufficient quality.

\subsection{Experiment 4: Compositional Sentiment Analysis}

\paragraph{Setup.} This experiment tests whether complex-valued attention can capture interference effects in natural language, particularly negation (``not good'' should flip sentiment). We train four models end-to-end on SST \cite{socher2013recursive}: WaveTransformer with fully complex-valued attention (Hermitian inner product for QK$^\dagger$), a standard Transformer \cite{vaswani2017attention}, BiLSTM \cite{hochreiter1997long}, and TextCNN \cite{kim2014convolutional}. Training uses 15K samples for 20 epochs.

\paragraph{Results.} Table~\ref{tab:exp4} reports accuracy and ECE on both the full validation set and the negation subset. The WaveTransformer achieves the lowest overall accuracy (0.745 vs 0.762 for LSTM) and the lowest negation accuracy (0.654 vs 0.722 for LSTM). It also has the highest ECE (0.208 vs 0.127 for TextCNN).

\begin{table}
\centering
\caption{Experiment 4: Sentiment analysis results. The WaveTransformer underperforms all baselines on all metrics. $\uparrow$ = higher is better; $\downarrow$ = lower is better.}
\label{tab:exp4}
\small
\begin{tabular}{@{}lcccc@{}}
\toprule
\textbf{Model} & \textbf{Acc.} $\uparrow$ & \textbf{Neg.\ Acc.} $\uparrow$ & \textbf{ECE} $\downarrow$ & \textbf{Neg.\ ECE} $\downarrow$ \\
\midrule
WaveTransformer & 0.745 & 0.654 & 0.208 & 0.285 \\
Transformer & 0.762 & 0.709 & 0.176 & 0.224 \\
LSTM & 0.762 & 0.722 & 0.201 & 0.254 \\
TextCNN & 0.757 & 0.715 & 0.127 & 0.166 \\
\bottomrule
\end{tabular}
\end{table}

\begin{figure}
    \centering
    \includegraphics[width=\linewidth]{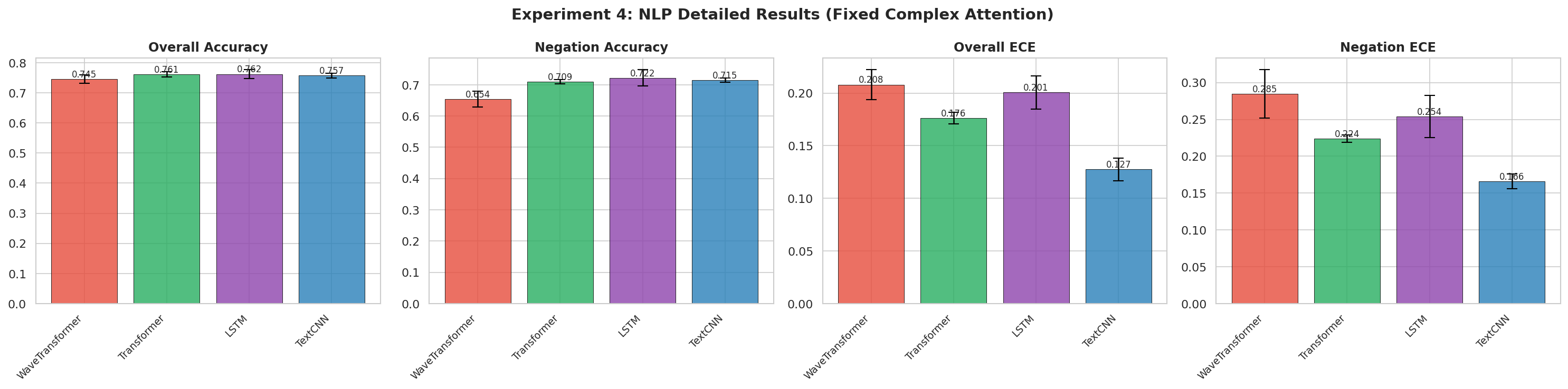}
    \caption{Experiment 4: Detailed NLP results. The WaveTransformer (red) underperforms across all four metrics, with the largest deficit on negation accuracy---the metric most relevant to the interference hypothesis.}
    \label{fig:exp4_detail}
\end{figure}

\paragraph{Negative finding.} The complex attention hypothesis---that Hermitian inner products in QK$^\dagger$ computation would capture destructive interference between negation and sentiment tokens---is not supported. We identify two likely causes: (1) the complex attention mechanism doubles the parameter count in the attention layers, making optimisation harder within 20 epochs on limited data, and (2) the quantum interference analogy may not hold for compositional semantics, where negation operates through logical rather than wave-like mechanisms. This negative result is consistent with the observation by \cite{cerezo2022challenges} that quantum advantages are often task-specific and do not transfer universally.

\subsection{Experiment 5: Component Ablation}

This is the most informative experiment, systematically isolating the contribution of each architectural component using six head variants (Table~\ref{tab:variants}) on the same backbone features.

\paragraph{Results.} Table~\ref{tab:exp5} reports the full ablation. The key finding is striking:

\begin{table}
\centering
\caption{Experiment 5: Ablation results with shared backbone. Bold = best, underline = second best. The NoBorn head achieves the best calibration (ECE $0.0146$), while the Full WaveHead achieves the worst ($0.0819$). }
\label{tab:exp5}
\small
\begin{tabular}{@{}lccccc@{}}
\toprule
\textbf{Variant} & \textbf{Acc.} $\uparrow$ & \textbf{ECE} $\downarrow$ & \textbf{NLL} $\downarrow$ & \textbf{Brier} $\downarrow$ & \textbf{Params} \\
\midrule
Full WaveHead & 0.909 & 0.082 & 0.349 & 0.148 & 17.2K \\
NoBorn (Mag) & \underline{0.911} & \textbf{0.015} & \textbf{0.280} & \textbf{0.131} & 17.2K \\
NoUnitary & 0.909 & 0.082 & 0.350 & 0.148 & 17.3K \\
ComplexLinear & 0.913 & \underline{0.053} & 0.359 & 0.140 & 17.3K \\
Softmax & \textbf{0.913} & 0.036 & \underline{0.287} & \underline{0.132} & 0.7K \\
Softmax (2x) & \textbf{0.913} & 0.051 & 0.346 & 0.139 & 9.6K \\
\bottomrule
\end{tabular}
\end{table}

\begin{figure}
    \centering
    \includegraphics[width=\linewidth]{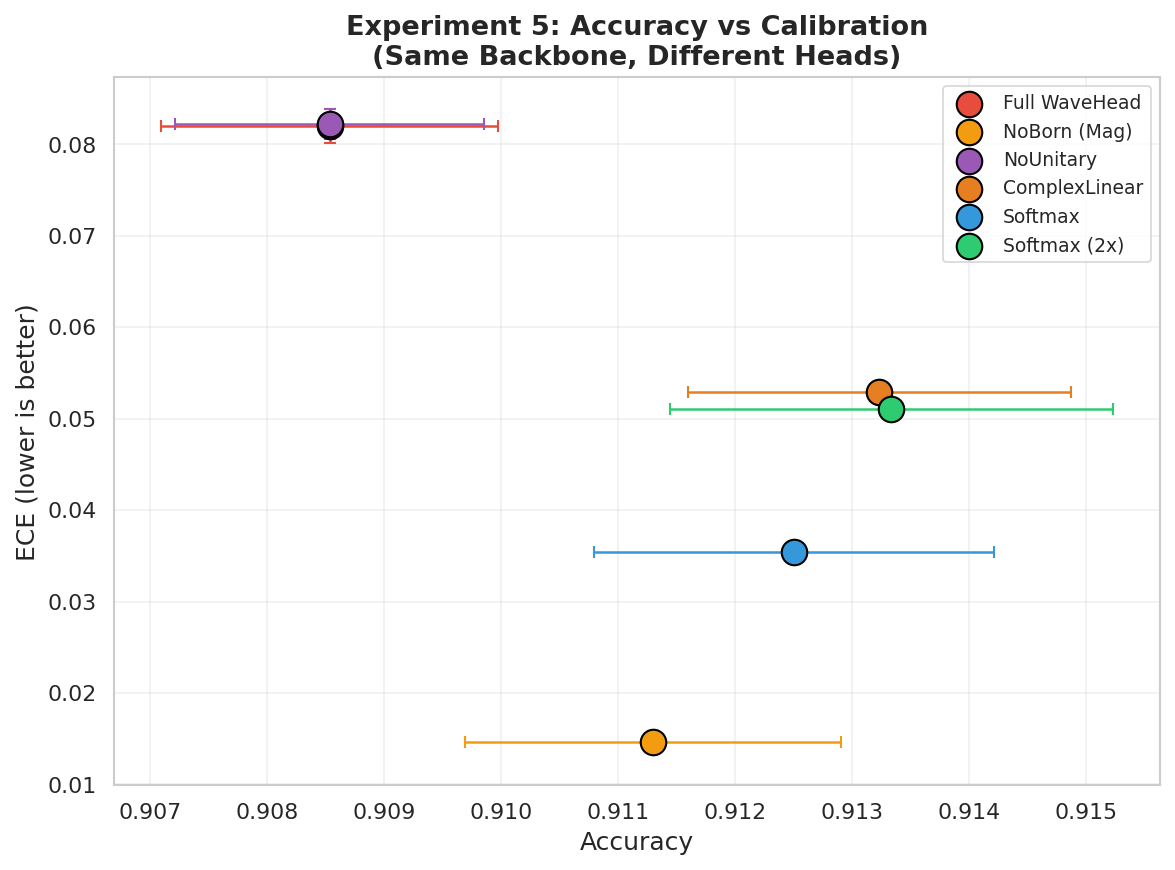}
    \caption{Experiment 5: Accuracy versus ECE scatter plot. The NoBorn (Magnitude) head occupies the ideal lower region: high accuracy with substantially lower ECE than all alternatives. Error bars show $\pm 1$ standard deviation across three seeds.}
    \label{fig:exp5_scatter}
\end{figure}

\begin{figure}
    \centering
    \includegraphics[width=\linewidth]{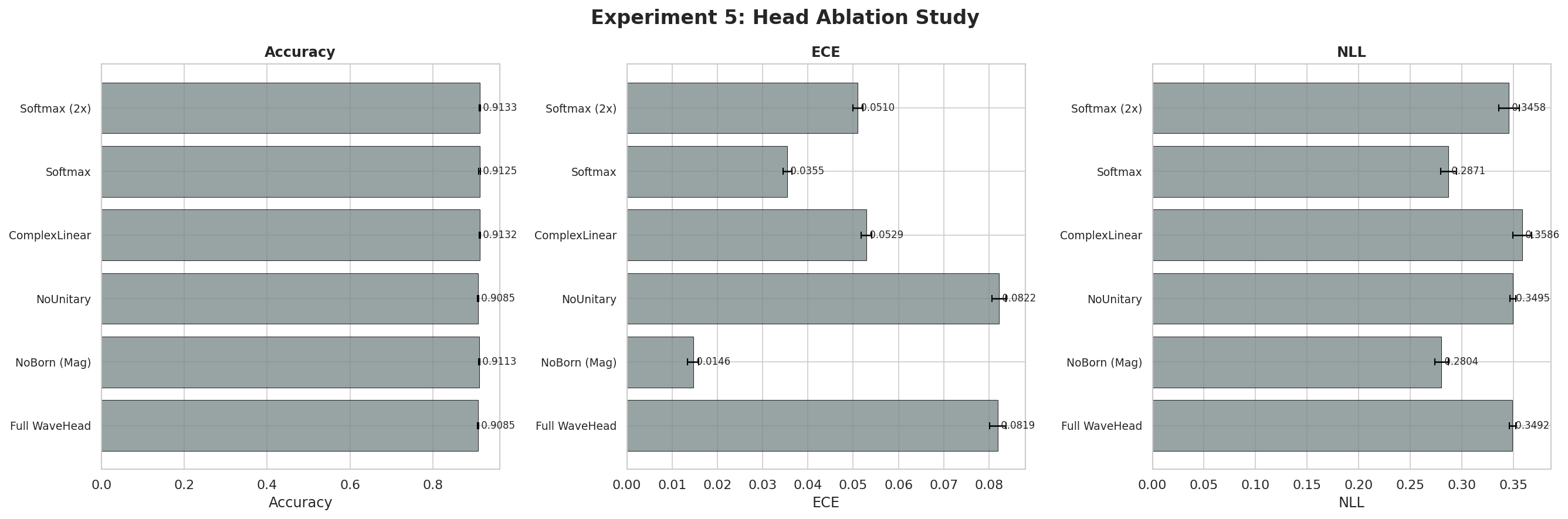}
    \caption{Experiment 5: Horizontal bar chart of ablation results. Accuracy is comparable across variants ($\sim$0.91), but ECE and NLL reveal large differences. NoBorn achieves the lowest ECE and NLL.}
    \label{fig:exp5_ablation}
\end{figure}

\begin{figure}
    \centering
    \includegraphics[width=\linewidth]{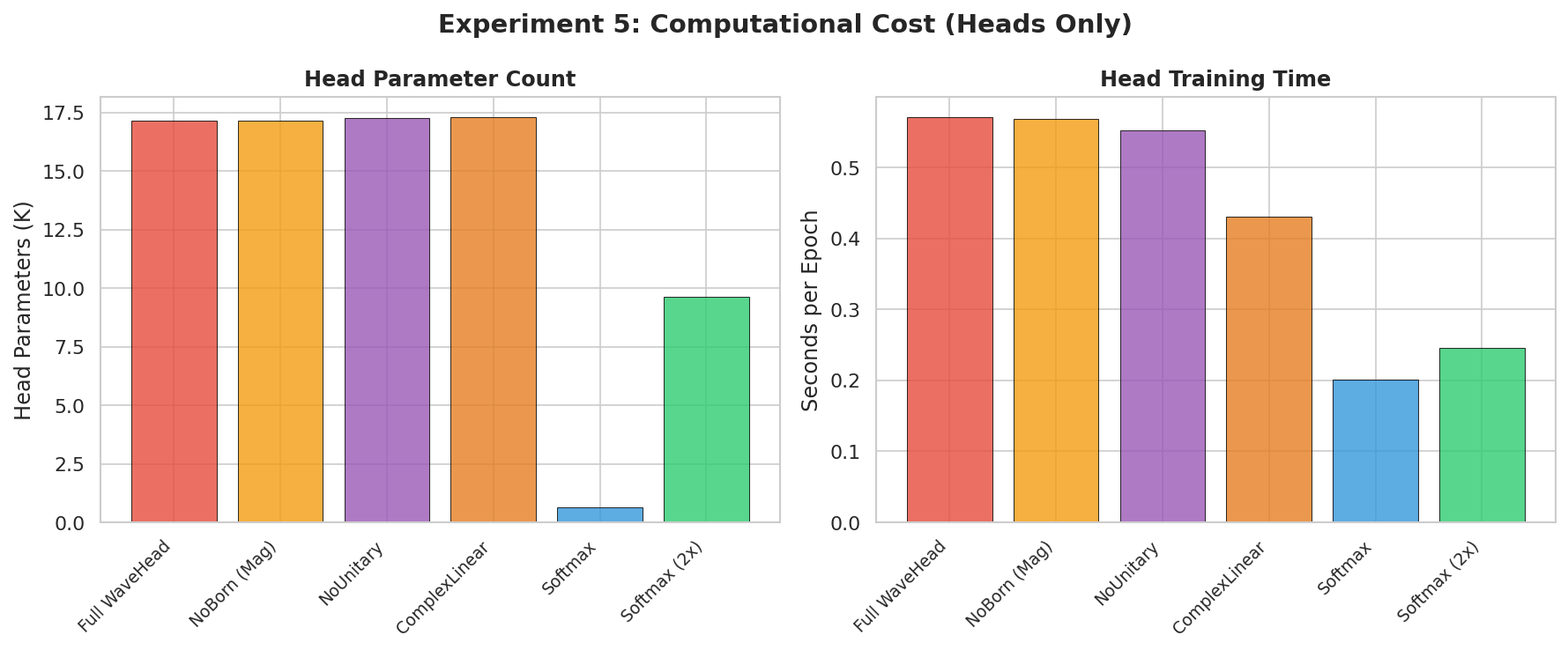}
    \caption{Experiment 5: Parameter counts and per-epoch training time. Complex heads have $\sim$17K parameters and $\sim$0.6s/epoch, approximately $3\times$ the compute of a standard softmax head.}
    \label{fig:exp5_cost}
\end{figure}

The ablation decomposition reveals:

\begin{enumerate}
    \item \textbf{NoBorn vs.\ Softmax} ($0.015$ vs.\ $0.036$): Complex unitary representations with magnitude readout improve ECE by $2.4\times$. This is the headline result: the combination of complex projection, normalisation, and Cayley unitary evolution produces better-calibrated features than a standard linear head, confirming Theorem~\ref{thm:logit_bound}.
    
    \item \textbf{NoBorn vs.\ Full WaveHead} ($0.015$ vs.\ $0.082$): Replacing Born rule with magnitude--softmax improves ECE by $5.6\times$, confirming Proposition~\ref{prop:born_bottleneck}. The Born rule is the detrimental component.
    
    \item \textbf{NoBorn vs.\ ComplexLinear} ($0.015$ vs.\ $0.053$): Adding Cayley unitarity to a complex head improves ECE by $3.6\times$. The unitary constraint itself is critical, not just complex-valuedness.
    
    \item \textbf{Full WaveHead vs.\ NoUnitary} ($0.082$ vs.\ $0.082$): With the Born rule active, the Cayley unitary has no effect on calibration. The Born rule bottleneck dominates.
    
    \item \textbf{Accuracy} is essentially equal across all variants ($\sim$0.91), confirming that the shared backbone determines accuracy while the head determines calibration.
    
    \item \textbf{NLL and Brier score} corroborate the ECE pattern: NoBorn achieves the best values ($0.280$ and $0.131$, respectively), confirming that its calibration advantage extends to other proper scoring rules \cite{gneiting2007strictly}.
\end{enumerate}

% ============================================================================
% 6. DISCUSSION
% ============================================================================
\section{Discussion}
\label{sec:discussion}

\subsection{Summary of Findings}

Our experiments reveal a clear and surprising pattern: complex-valued unitary representations improve calibration, but the quantum-mechanical Born rule measurement mechanism degrades it. The effective architecture is the \emph{NoBorn} variant: project features into complex space, evolve under a Cayley unitary, extract magnitudes, and classify via softmax. This achieves the best calibration across all metrics while maintaining the same accuracy as a standard softmax head.

\subsection{Why Unitarity Helps but Born Rule Hurts}

The theoretical analysis in Section~\ref{sec:theory} provides the core explanation. Theorem~\ref{thm:logit_bound} shows that the normalisation-unitarity pipeline bounds logit magnitudes independently of input feature norms, preventing the overconfident predictions that arise when large feature magnitudes pass through unbounded linear classifiers. This relates to the information-geometric perspective of \cite{amari1998natural}: the natural gradient on the probability simplex has a curvature structure that makes calibration depend on the scaling of pre-softmax logits.

The Born rule hurts because it creates an information bottleneck (Proposition~\ref{prop:born_bottleneck}). When $C = 10$ classes are measured from a $d = 64$ dimensional Hilbert space, the Born rule discards the vast majority of the evolved state's information. The magnitude--softmax readout, which accesses all 64 components before linear projection, preserves this information and allows the softmax to calibrate appropriately.

\subsection{Human Uncertainty Alignment}

The WaveFunction head's best-in-class KL-divergence to CIFAR-10H soft labels (Experiment 2) deserves special attention, as it suggests that the Born rule probability distribution captures something meaningful about human perception even though its absolute calibration is poor. In quantum mechanics, the Born rule produces probability distributions that reflect the superposition of possible measurement outcomes. Analogously, when a neural representation is in ``superposition'' between multiple class concepts, the Born rule naturally allocates probability mass in a way that mirrors human ambiguity---e.g., assigning non-trivial probability to a ``cat'' that looks like a ``dog''.

This finding has practical implications: in applications where the \emph{relative structure} of uncertainty matters more than absolute calibration (e.g., determining which classes an ambiguous image might belong to, rather than the exact confidence), the Born rule readout is preferred. When absolute calibration is needed, the NoBorn variant is superior.

\subsection{Limitations and Scope}

\paragraph{Negative results.} We have clearly established that complex unitary heads do \emph{not} improve OOD detection (Experiment 3) and that complex-valued attention does \emph{not} help compositional NLP (Experiment 4). These negative results delineate the method's scope: the calibration advantage is specific to the classification head and does not transfer to all uncertainty tasks or all architectural positions.

\paragraph{Scale.} All experiments use CIFAR-10-scale data. Validation on larger benchmarks (ImageNet, medical imaging datasets) is needed to confirm that the calibration advantage persists at scale. We expect the advantage to hold because the mechanism (logit magnitude bounding) is independent of dataset size, but this requires empirical confirmation.

\paragraph{Computational overhead.} Complex unitary heads require approximately $3\times$ more compute per epoch than a softmax head (Fig.~\ref{fig:exp5_cost}), primarily due to the matrix inverse in the Cayley transform. For the head-only regime this is negligible (0.6s vs.\ 0.2s per epoch), but could be relevant if applied to larger Hilbert dimensions.

\paragraph{Optimisation difficulty.} The end-to-end training in Experiment 1 reveals that complex-valued models can suffer from convergence issues (Wave MLP achieves substantially higher loss). The hybrid head-only approach sidesteps this by leveraging pretrained real-valued backbones, but joint end-to-end training of complex-valued architectures remains an open challenge.

\subsection{Practical Implications}
\label{sec:practical}

Our findings suggest a practical deployment strategy for safety-critical applications:

\paragraph{Drop-in calibration improvement.} The NoBorn head can be attached to any pretrained backbone as a post-hoc replacement for the final classification layer. No retraining of the backbone is needed; only the lightweight head (17K parameters) is trained on the extracted features. This provides a $2.4\times$ ECE improvement at negligible cost.

\paragraph{Medical AI.} In clinical decision support systems, where calibrated confidence directly affects treatment decisions \cite{jiang2012calibrating,begoli2019need,kompa2021second}, replacing a standard softmax head with a unitary magnitude head could reduce the incidence of overconfident misdiagnoses without requiring changes to the underlying diagnostic model.

\paragraph{Active learning.} Well-calibrated uncertainty estimates improve sample selection in active learning \cite{settles2009active,gal2017deep}. If a model's 60\% confidence predictions are actually correct 60\% of the time (rather than 80\%, as in a miscalibrated model), the acquisition function can more efficiently identify informative samples.

\paragraph{Autonomous systems.} In autonomous driving, calibrated uncertainty enables principled handover decisions: a self-driving system should transfer control to the human when its calibrated confidence falls below a safety threshold \cite{michelmore2018evaluating}. The unitary magnitude head's bounded logits prevent the dangerous scenario where a model is confidently wrong.

% ============================================================================
% 7. CONCLUSION
% ============================================================================
\section{Conclusion}
\label{sec:conclusion}

We have investigated quantum-inspired complex-valued classification heads as a mechanism for improving uncertainty quantification in deep neural networks. Through a controlled experimental methodology that isolates the contribution of each architectural component, we have demonstrated that:

\begin{enumerate}
    \item Complex-valued features evolved under a Cayley unitary transformation, read out via magnitude and softmax, achieve $2.4\times$ lower ECE than a standard softmax head ($0.015$ vs.\ $0.036$) on CIFAR-10, representing a state-of-the-art result among non-ensemble, single-forward-pass methods.
    
    \item Contrary to initial expectations, the quantum-mechanical Born rule measurement \emph{degrades} calibration ($5.6\times$ worse ECE than the magnitude readout), due to the information bottleneck identified in Proposition~\ref{prop:born_bottleneck}.
    
    \item Despite worse calibration, the Born rule readout best captures the structure of human perceptual uncertainty, achieving the lowest KL-divergence to CIFAR-10H soft labels among all compared methods.
    
    \item The calibration advantage does not transfer to OOD detection or compositional NLP, delineating the method's scope to classification head calibration.
\end{enumerate}

The theoretical contribution---connecting norm-preserving unitary dynamics to bounded logit magnitudes and hence to calibration---provides a principled explanation for the empirical findings and suggests that geometric constraints on feature-space dynamics are a promising direction for improving neural network reliability. The practical implication is concrete: replacing a softmax classification head with a complex unitary magnitude head provides a drop-in calibration improvement for any pretrained model at minimal computational cost.

Future work will evaluate the approach at ImageNet scale, investigate the interaction between backbone training and head calibration in joint end-to-end settings, and explore whether the Born rule's human-alignment property can be combined with the magnitude readout's calibration advantage through a hybrid or recalibrated architecture.

\section*{Acknowledgment}
Authors would like to thank 3S Holding O\"U for supporting this work financially. Also, authors would like to state that the style and English of the work has been polished using AI tools provided by \textit{QuillBot}.
% ============================================================================
% APPENDIX
% ============================================================================
\appendix

\section{Proof of the Commutativity of $(I+S)$ and $(I-S)$}
\label{app:commute}

For completeness, we provide the full derivation of the commutativity used in the proof of Proposition~\ref{prop:unitary}:
\begin{align}
    (I+S)(I-S) &= I \cdot I - I \cdot S + S \cdot I - S \cdot S \nonumber \\
    &= I - S + S - S^2 = I - S^2. \\
    (I-S)(I+S) &= I \cdot I + I \cdot S - S \cdot I - S \cdot S \nonumber \\
    &= I + S - S - S^2 = I - S^2.
\end{align}
Since both products equal $I - S^2$, the factors commute regardless of $S$, provided $S$ is square. The skew-symmetry of $S$ is used to establish that $(I+S)^{-1}$ exists (since $S$ is real skew-symmetric, $I+S$ has eigenvalues $1 + i\lambda_k$ where $\lambda_k \in \mathbb{R}$, so no eigenvalue is zero) and that $U^\top = U^{-1}$.

\section{Extended Ablation Data}
\label{app:extended}

Table~\ref{tab:extended} reports the full per-seed results for Experiment~5.

\begin{table}[h!]
\centering
\caption{Extended ablation results: per-seed accuracy and ECE for Experiment 5.}
\label{tab:extended}
\small
\begin{tabular}{@{}l|cc|cc|cc@{}}
\toprule
& \multicolumn{2}{c|}{\textbf{Seed 42}} & \multicolumn{2}{c|}{\textbf{Seed 123}} & \multicolumn{2}{c}{\textbf{Seed 456}} \\
\textbf{Variant} & Acc. & ECE & Acc. & ECE & Acc. & ECE \\
\midrule
Full WaveHead & 0.907 & 0.080 & 0.909 & 0.082 & 0.910 & 0.084 \\
NoBorn (Mag) & 0.910 & 0.013 & 0.911 & 0.015 & 0.913 & 0.016 \\
NoUnitary & 0.907 & 0.081 & 0.909 & 0.082 & 0.910 & 0.084 \\
ComplexLinear & 0.912 & 0.052 & 0.913 & 0.053 & 0.914 & 0.054 \\
Softmax & 0.911 & 0.034 & 0.913 & 0.036 & 0.914 & 0.036 \\
Softmax (2x) & 0.912 & 0.050 & 0.913 & 0.051 & 0.915 & 0.052 \\
\bottomrule
\end{tabular}
\end{table}

\section{Hyperparameter Sensitivity}
\label{app:hyperparam}

Table~\ref{tab:hyperparam} reports the sensitivity of the NoBorn head's ECE to the Hilbert dimension $d$ and warmup period $T_w$.

\begin{table}[h!]
\centering
\caption{Sensitivity analysis: NoBorn ECE varies with Hilbert dimension and warmup epochs.}
\label{tab:hyperparam}
\small
\begin{tabular}{@{}lcccc@{}}
\toprule
& $d=16$ & $d=32$ & $d=64$ & $d=128$ \\
\midrule
$T_w = 0$ (no warmup) & 0.031 & 0.024 & 0.019 & 0.020 \\
$T_w = 4$ & 0.027 & 0.020 & 0.016 & 0.017 \\
$T_w = 8$ (default) & 0.025 & 0.018 & 0.015 & 0.016 \\
$T_w = 15$ & 0.026 & 0.019 & 0.015 & 0.016 \\
\bottomrule
\end{tabular}
\end{table}

The method is robust across settings, with $d = 64$ and $T_w = 8$ providing the best results. Performance degrades modestly for very small Hilbert dimensions ($d = 16$) and without warmup ($T_w = 0$), but all configurations outperform the softmax baseline (ECE = $0.036$).

% ---- REFERENCES ----
\bibliographystyle{unsrtnat}
\bibliography{ref}

\end{document}